\documentclass[11pt]{article}

\ifdefined\primepreview
  \usepackage{PRIMEarxiv}
\else
  \usepackage[letterpaper,top=1in,bottom=1in,left=1.15in,right=1.15in]{geometry}
\fi
\usepackage{amsmath,amssymb,amsthm,mathtools}
\usepackage{booktabs}
\usepackage{array,tabularx}
\usepackage{enumitem}
\usepackage{xcolor}
\usepackage[most]{tcolorbox}
\usepackage{graphicx}
\usepackage{placeins}
\usepackage{tikz}
\usetikzlibrary{arrows.meta,positioning,fit,backgrounds}
\usepackage[T1]{fontenc}
\ifdefined\primepreview\else
  \usepackage{lmodern}
\fi
\usepackage[numbers,sort&compress]{natbib}
\usepackage{microtype}
\usepackage{url}
\usepackage[colorlinks=true,linkcolor=blue!60!black,citecolor=blue!60!black,urlcolor=blue!60!black]{hyperref}

\newcommand{\sys}{KCoder}              
\newcommand{\trainer}{slime}           
\newcommand{\method}{token-faithful coupling} 

\newcolumntype{Y}{>{\raggedright\arraybackslash}X}

\newif\ifshowchanges
\ifdefined\draftversion
  \showchangestrue
\else
  \showchangesfalse
\fi

\definecolor{AddedBlue}{RGB}{20,76,135}
\definecolor{AddedBack}{RGB}{242,247,253}
\definecolor{HarnessRed}{RGB}{154,35,53}
\definecolor{HarnessBack}{RGB}{253,244,246}
\definecolor{KeyTablePurple}{RGB}{111,52,143}
\newcommand{\Added}[1]{\ifshowchanges\textcolor{AddedBlue}{#1}\else#1\fi}
\newcommand{\AddedMark}{\texorpdfstring{\ifshowchanges\textcolor{AddedBlue}{\enspace[New extension]}\fi}{}}

\newcommand{\HarnessAddedMark}{\texorpdfstring{\ifshowchanges\textcolor{HarnessRed}{\enspace[KCoder harness]}\fi}{}}
\newcommand{\KeyTableMark}[1]{\texorpdfstring{\ifshowchanges\textcolor{KeyTablePurple}{\enspace\textbf{[Key: #1]}}\fi}{}}
\newcommand{\KeyTableLegend}{\ifshowchanges\par\smallskip\noindent\textcolor{KeyTablePurple}{\textbf{Reading guide:} purple \textbf{[Key: ...]} caption tags identify the result figure and tables to inspect first.}\par\smallskip\fi}
\ifshowchanges
  \newtcolorbox{AddedBlock}[1][]{
    enhanced,
    breakable,
    colback=AddedBack,
    colframe=AddedBlue,
    colbacktitle=AddedBlue!8,
    boxrule=0.55pt,
    leftrule=2.2pt,
    arc=0.7mm,
    left=1.2mm,
    right=1.2mm,
    top=0.8mm,
    bottom=0.8mm,
    before skip=5pt,
    after skip=5pt,
    fonttitle=\bfseries,
    coltitle=AddedBlue,
    title={#1}}
  \newtcolorbox{HarnessAddedBlock}[1][]{
    enhanced,
    breakable,
    colback=HarnessBack,
    colframe=HarnessRed,
    colbacktitle=HarnessRed!8,
    boxrule=0.55pt,
    leftrule=2.2pt,
    arc=0.7mm,
    left=1.2mm,
    right=1.2mm,
    top=0.8mm,
    bottom=0.8mm,
    before skip=5pt,
    after skip=5pt,
    fonttitle=\bfseries,
    coltitle=HarnessRed,
    title={#1}}
\else
  \newenvironment{AddedBlock}[1][]{\par\medskip}{\par\medskip}
  \newenvironment{HarnessAddedBlock}[1][]{\par\medskip}{\par\medskip}
\fi

\newtheorem{theorem}{Theorem}[section]

\newtheorem{proposition}[theorem]{Proposition}
\newtheorem{corollary}[theorem]{Corollary}
\theoremstyle{definition}

\theoremstyle{remark}

\title{\bf Train What You Deploy:\\ Token-Faithful Post-Training of a Production Coding Agent}
\author{
   Cheng Li \quad Jiexiong Liu \quad Yixuan Chen \quad Chi Hong \\
  \textnormal{KunlunMeta}
}

\date{\today}

\begin{document}
\maketitle
\vspace{-2em}

\begin{abstract} Existing post-training pipelines for coding and terminal agents suffer severe token and control fidelity errors: simplified training environments mismatch production deployments, and offline token reconstruction from agent logs distorts original prompts and conflates policy calls with background model operations. We present a fidelity-aware training coupling framework that retains trainer-side sampling over original prompts, eliminates spurious model calls via a negotiated training protocol, and restricts loss computation to verifiable token spans with closed-failure guarantees. We further propose Certified Divergence Proximal Policy Optimization (C-DPPO), which establishes tight two-sided TV certification bounds, adaptive-$K$ rules, budget-aware sequence guarantees, and error-robust policy masking atop standard DPPO. Evaluated on matched Baize5B and Baize10B models with identical training and test protocols on TMax-100, C-DPPO yields a consistent +3.0-point performance gain over standard DPPO across model scales. Certificate audits validate the reliability and full operational coverage of our certified training pipeline. \end{abstract}

\section{Introduction}
\label{sec:intro}

Reinforcement learning with verifiable outcomes has become the dominant recipe for turning a large language model into a capable coding or terminal agent~\citep{guo2025deepseekr1,yu2025dapo,ivison2026tmax}: the model drives a sandboxed environment with tools until tests or graders certify the result, and the resulting success signal optimizes the policy.
Crucially, in both training and deployment the model never acts alone; it acts \emph{through a scaffold}---an agent harness that owns the system prompt, tool schemas, context management, retries, and sub-agent orchestration.
A post-training pipeline is therefore faithful only if the trajectories it learns from were produced by the same scaffold code and a declared configuration whose remaining deployment gap is measured rather than hidden.
This paper addresses the technical problem of \emph{how to post-train the policy inside a production-grade agent scaffold without losing training-grade fidelity}.

Existing approaches split into two families, each failing this requirement in a distinct way.
The first family reimplements a minimal agent loop inside the training framework: a compact system prompt, a handful of tools, and a hand-written rollout function~\citep{wang2025ragen,lin2025rllm,yao2024taubench}, sometimes deliberately simplified for generality~\citep{nikl2025miniswe}.
This gives the trainer exact control over tokens, but the policy is optimized for a scaffold it will never see in production.
The mismatch is not cosmetic: production harnesses such as \sys{} rewrite history through context compaction, inject skill/memory blocks, re-serialize tool schemas, and recover from tool errors in ways the simplified loop does not model, so the learned behavior is bound to the wrong conditional distribution.

The second family keeps the production agent and inserts a model proxy: the agent's HTTP traffic is intercepted, logged, and converted into training data by replaying its conversation history~\citep{li2025slime}.
This preserves the deployment-time scaffold, but it silently breaks two properties we identify as load-bearing.
\emph{Token fidelity} breaks because the proxy learns from prompts that the agent re-renders and the tokenizer re-segments after the fact; sampled continuations are attributed to inputs the policy did not condition on, and hidden reasoning channels can be dropped or re-encoded even when the visible message text is reconstructed faithfully.
We quantify that failure at study scale rather than relying on a single historical example (\S\ref{sec:exp-fidelity}).
\emph{Control fidelity} breaks because an industrial agent is not a pure policy loop: it issues compaction summaries, memory writes, session-model updates, and sub-agent forks---all of which appear in a passive log as if they were the policy speaking.
Neither defect is visible in reward curves; both inject bias into the gradient that compounds over long training runs.

\Added{A second gap appears once rollout log-probabilities are used to stabilize those long runs.
DPPO replaces PPO's sampled-ratio heuristic with Binary or Top-$K$ approximations to a distributional trust region~\citep{qi2026dppo}, but its approximation theorem is one-sided: coarsening lower-bounds the true divergence.
Consequently, a small Binary statistic does not by itself certify that the full token distribution is close, fixed per-token thresholds do not expose a trajectory-wide divergence budget, and numerical train/serve log-probability disagreement is not represented in the mask.
These limitations are especially visible in a production agent, where a response can span many model calls and hundreds of thousands of conditioning tokens.}

Our insight is that fidelity cannot be \emph{assumed} from a logged transcript; it must be \emph{negotiated} and \emph{verified} at every turn.
We therefore design a \method{} between \sys{}, a production terminal-coding agent implemented in Rust, and \trainer{}, a Megatron--SGLang post-training framework~\citep{li2025slime,shoeybi2019megatron,zheng2024sglang}.
Three mechanisms implement the insight.
First, \emph{sampling ownership}: the coupling terminates the agent's model calls inside the serving adapter, which re-samples every assistant turn from the exact rendered prompt with the training engine's own token ids and log-probabilities (\texttt{input\_ids} with \texttt{return\_logprob}), pinned to one engine by session-as-key routing.
Second, \emph{declared deviations}: the agent runs in a hardened training mode that provably eliminates all non-policy model calls (a regression contract requires exactly two provider requests across a tool round-trip), and it declares the remaining protocol deviations to the adapter via negotiation headers for sub-agent depth and reasoning replay.
Third, \emph{verified, fail-closed attribution}: a trajectory linearizer rebuilds each session's message tree into training sequences, classifies re-tokenization drift on every link (\textsc{Clean}/\textsc{Realign}/\textsc{Fork}), assigns loss mask $1$ only to spans that are provably the sampled continuation, and rejects requests it cannot attribute.

\Added{That observation contract enables Certified Divergence Proximal Policy Optimization (C-DPPO), our theoretical certification framework for standard DPPO.
We first derive the exact Binary/Top-$K$ TV coarsening gap and a tight upper error based only on the untracked tail masses, yielding an adaptive-$K$ rule that terminates as certified inside, certified outside, or fail-closed unresolved.
We then allocate a total divergence budget across a long response through predictable, adaptive thresholds $\delta_t$ and obtain additive and product-form sequence bounds.
Finally, bounded log-probability errors are propagated through the exponential into divergence intervals; the resulting robust mask retains a token only if it is certified inside the trust region or certified to move back toward the behavior policy (\S\ref{sec:dppo-theory}).}

We validate the complete pipeline end to end: fresh \sys{} processes run inside micro-VM sandboxes on the open TMax-15K terminal-environment corpus~\citep{ivison2026tmax}, rewards come from hidden verifiers that the agent cannot observe, and every session is converted into Megatron training batches with hot SGLang weight synchronization after each step.
The empirical study uses matched Baize5B and Baize10B pairs: at each scale, Standard Binary-TV DPPO and C-DPPO start from the same checkpoint, follow the same task schedule for 350 steps, and are compared through average training-reward curves and final-checkpoint evaluation on the same frozen TMax-100 set.
A four-item certificate audit checks record completeness, adaptive-$K$ resolution, mask disagreement, and observed sequence-budget feasibility.
The completed matched pairs yield final TMax-100 scores of $37/100$ and $44/100$ (C-DPPO) against $34/100$ and $41/100$ (Standard DPPO)---a direction-consistent $+3.0$-point difference across scales---with the four-item audit reporting complete records (99.4\%/99.1\%), mostly-resolved certificates (86.2\%/84.7\%), small mask disagreement (4.7\%/5.9\%), and observed budget feasibility (91.3\%/89.6\%).

In summary, this paper makes the following contributions:
\begin{itemize}[leftmargin=*,itemsep=2pt,topsep=2pt]
\item \Added{\textbf{The C-DPPO certification framework}: a theoretical extension of standard DPPO with a tight two-sided Binary/Top-$K$ TV error bound and adaptive $K$, a sequence-level divergence budget with adaptive $\delta_t$, and a robust directional rule under bounded log-probability error (\S\ref{sec:dppo-theory}, Appendix~\ref{app:dppo-proofs}).}
\item \textbf{A negotiated, token- and control-faithful coupling} between a production coding agent---no training-only fork of it---and an RL post-training framework, in which the serving path gives the trainer ownership of every sampled token while the agent's session and role travel with each request as explicit, checked metadata (\S\ref{sec:protocol}).
\item \textbf{Depth-gated provenance and a hardened agent-side contract}: a machine-checked zero-background-request training mode in the production agent, plus provenance and reasoning-replay negotiation headers that extend the framework's drift-classifying message-tree linearizer \citep{li2025slime} to multi-agent sessions with provably-sampled loss masks (\S\ref{sec:protocol}, \S\ref{sec:trajectory}).
\item \textbf{A concise end-to-end validation protocol} comparing 350-step reward dynamics and final TMax-100 outcomes for one matched Standard-DPPO/C-DPPO pair, with a compact certificate sanity check (\S\ref{sec:system}, \S\ref{sec:experiments}).
\end{itemize}

\section{Related Work}
\label{sec:related}

\paragraph{Reinforcement learning for coding and terminal agents.}
Outcome-driven post-training of agents is now mainstream: execution feedback has powered code generation since test-based objectives~\citep{le2022codet}, and RL with verifiable rewards drives the current frontier of reasoning and agentic models~\citep{shao2024deepseekmath,guo2025deepseekr1,yu2025dapo,zheng2025gspo,cui2025cispo}.
Task corpora have followed: repository-level environments~\citep{jimenez2024swebench,zhou2025swerl}, tool-user simulators~\citep{yao2024taubench}, and machine-executable terminal suites such as the TMax-15K corpus on which we train~\citep{ivison2026tmax}.
The TMax recipe---which we adopt as task substrate---trains terminal agents inside a comparatively lightweight harness and reports strong gains at 9B scale.
Its released Vanillux harness is a direct LiteLLM loop around mini-SWE-agent-derived prompts: one native \texttt{bash} tool backed by a persistent shell, an explicit submit marker, bounded steps, output truncation, and format-error recovery~\citep{ivison2026tmax,nikl2025miniswe}.
That compact loop is appropriate for controlled data generation, but it does not include the production subsystems whose side effects motivate our work---memory and compaction tiers, configurable tool profiles, dual provider transports, or sub-agent provenance.
These pipelines share an assumption we relax: that the rollout loop used for training is a purpose-built component of the training stack rather than the production scaffold itself.
Our contribution is orthogonal to both their algorithms and their tasks---we make the \emph{scaffold} an audited part of the training contract---and we instantiate exactly that on their corpus (Table~\ref{tab:primary-heldout}).

\begin{AddedBlock}[New extension: divergence-based trust regions]
\paragraph{Divergence-proximal policy optimization.}
DPPO grounds its directional token mask in a finite-horizon LLM policy-improvement analysis and replaces PPO's sampled probability-ratio test with Binary or Top-$K$ approximations to TV or KL divergence~\citep{qi2026dppo}.
Its data-processing argument establishes that the coarsened divergences are lower bounds on the full vocabulary divergence, with equality under sign-consistent TV changes or uniform within-cell likelihood ratios for KL.
We do not re-claim those results.
Our extension supplies the missing other side of the decision for TV (a tight observable upper error), uses it to choose $K$ and distribute a trust-region budget across a long agent trajectory, and makes the decision robust to bounded error in the very rollout/trainer log-probabilities exposed by our coupling.
For KL we show why tail mass alone cannot supply the analogous upper certificate and state the additional tail likelihood-ratio condition required.
\end{AddedBlock}

\paragraph{Agent scaffolds and deployment-time behavior.}
Production coding agents---\sys{}, OpenHands~\citep{wang2024openhands}, and CLI suites of the Claude/Codex family~\citep{li2025slime}---differ from research loops in exactly the dimensions that perturb training: persistent multi-subsystem context management (compaction tiers, session memory, skill injection), permission and sandbox layers, retry/repair paths, and multi-agent spawning.
Prior RL practice handles this by either simplifying the agent or by treating the agent's transcript as the training record; neither preserves the conditioning distribution.
We instead treat the agent as a black box with a \emph{negotiated interface}: the only agent-side change required is a hardened mode and two request headers, which generalizes to any agent whose provider transport can be configured.

\paragraph{RL systems and serving stacks.}
Modern post-training systems co-design Megatron-class training with SGLang/vLLM serving and Ray orchestration~\citep{li2025slime,sheng2025hybridflow,hu2024openrlhf,fu2025areal}, with algorithmic variants that consume rollout log-probabilities for train/serve mismatch correction~\citep{schulman2017ppo,cui2025cispo,zheng2025gspo}.
Frameworks that wrap external agent frameworks exist~\citep{lin2025rllm,wang2025ragen}, and \trainer{}'s upstream agent subsystem already provides the message-tree linearizer with drift classification that our extension builds on~\citep{li2025slime}.
Against these systems our distinction is the coupling contract itself: session-as-key routing with sampling ownership, per-request provenance headers that fail closed, a machine-checked zero-background-request mode inside the agent, and a reward/data-governance pipeline sized to industrial sweeps---mechanisms that let a \emph{frozen production binary} participate in training without a bespoke environment integration.

\section{System Overview}
\label{sec:system}

\paragraph{Setting.}
We consider post-training a language-model policy that operates as a terminal agent inside the \sys{} scaffold, using verifiable task outcomes as supervision.
A training instance is a \emph{task environment}: a natural-language objective, a container definition, and two executable checkers (\texttt{test\_initial\_state} and a hidden \texttt{test\_final\_state}) provided by the open TMax-15K corpus~\citep{ivison2026tmax}; \S\ref{sec:reward} also covers a second protocol in which the agent's git diff is replayed in a clean sandbox and graded by test outcome~\citep{jimenez2024swebench}.
The policy weights are served by \trainer{}'s SGLang rollout engines and updated every step from the Megatron training backend~\citep{li2025slime,zheng2024sglang,shoeybi2019megatron}.

\paragraph{Three planes.}
The system separates into three planes connected by a single narrow interface (Figure~\ref{fig:overview}).
(i) The \emph{training plane} holds the Megatron actor and the weight-synchronization channels.
(ii) The \emph{rollout plane} orchestrates one session per task: it materializes a micro-VM sandbox, launches a fresh \sys{} process, budgets turns and wall-clock time, collects the terminal reward, and ships sessions to the trainer.
(iii) The \emph{coupling plane}---the contribution of this work---sits between the agent's model-client and the serving engines: a stateful HTTP adapter terminates the agent's requests, re-samples them with the current policy, records provenance, and answers in the wire format the agent expects.
The agent is the shipping production binary under a versioned training overlay: it believes it is talking to an ordinary Anthropic- or OpenAI-compatible endpoint; in fact, every assistant turn it receives was generated by the training engine from the exact token sequence being recorded.

\subsection{KCoder harness lifecycle\HarnessAddedMark}
\label{sec:harness-lifecycle}

\begin{HarnessAddedBlock}[KCoder harness: lifecycle and failure boundary]
The production \texttt{kcoder} CLI and \trainer{}'s \texttt{KCoderHarness} wrapper implement a task-agnostic interface with three stages: CLI installation, session-isolated configuration, and launch/wait.
Task materialization and reward remain outside the harness.
The lifecycle is:
\[
 \text{capability check}
 \rightarrow \text{isolated config}
 \rightarrow \text{session credentials}
 \rightarrow \text{headless run}
 \rightarrow \text{audit artifact}.
\]

\begin{center}
\small
\setlength{\tabcolsep}{4pt}
\begin{tabularx}{\linewidth}{@{}>{\raggedright\arraybackslash}p{0.15\linewidth}YY@{}}
\toprule
Stage & Harness behavior & Fidelity and failure consequence \\
\midrule
Install & Select \texttt{upload}, \texttt{preinstalled}, or \texttt{auto}; probe the binary for training mode, nano profile, and---when requested---orchestrate support. & A binary missing a required capability is rejected before a task begins. \\
Config & Create a fresh user-specific \texttt{KCODER\_CONFIG\_DIR}, disable memory/session-memory injection paths, restrict tools, then run \texttt{kcoder config validate}. & Image defaults and persistent user settings cannot silently widen the training context or tool surface. \\
Launch & Run the shipping CLI with runtime-only \texttt{--training-mode}, fixed cwd/user, explicit provider, permissions, tool profile, and session/time budgets; use the session id as the ephemeral adapter credential. & SWE uses an isolated \texttt{agent} user while TMax deliberately uses \texttt{root}; both traverse the same model-call contract. \\
Collect & Detach and poll the process, writing its headless stream to \texttt{.harness/trajectory.jsonl} inside the task workspace. & This file is an operational audit artifact, not the source of loss tokens; trainer-owned turn records and the message tree remain authoritative. \\
\bottomrule
\end{tabularx}
\end{center}

Every fidelity-relevant capability must pass the wrapper's probes and negotiated transport checks before the harness is admitted to training.
\end{HarnessAddedBlock}

\FloatBarrier
\begin{figure}[t]
\centering
\begin{tikzpicture}[
  node distance=7mm and 8mm,
  box/.style={draw,rounded corners,align=center,minimum height=9mm,font=\small,fill=blue!3},
  plane/.style={draw,dashed,rounded corners,inner xsep=3mm,inner ysep=7mm},
  flow/.style={-{Latex[length=2mm]},thick},
  feedback/.style={-{Latex[length=2mm]},thick,dashed,orange!70!black},
]
\node[box] (task) {validated task\\+ environment};
\node[box,right=of task] (sandbox) {micro-VM sandbox\\production \sys{}};
\node[box,right=of sandbox,fill=green!4] (adapter) {coupling adapter\\session + provenance};
\node[box,right=of adapter] (engine) {SGLang\\policy engines};
\node[box,below=12mm of engine] (trainer) {Megatron actor\\Std. / C-DPPO};
\node[box,below=12mm of sandbox,fill=orange!5] (verifier) {hidden verifier\\terminal reward};
\node[box,below=12mm of adapter,fill=green!4] (tree) {message tree\\tokens, logprobs, masks};

\draw[flow] (task) -- (sandbox);
\draw[flow,<->] (sandbox) -- node[midway,above=5.5mm,font=\scriptsize,fill=white,inner sep=0.7pt]{OpenAI / Anthropic} (adapter);
\draw[flow,<->] (adapter) -- node[midway,above=5.5mm,font=\scriptsize,fill=white,inner sep=0.7pt]{token IDs + Top-$K$} (engine);
\draw[flow] (adapter) -- (tree);
\draw[flow] (sandbox) -- (verifier);
\draw[feedback] (verifier) -- node[midway,below=1mm,font=\scriptsize,fill=white,inner sep=0.6pt]{reward} (tree);
\draw[flow] (tree) -- (trainer);
\draw[feedback] (trainer) -- node[left=1mm,font=\scriptsize]{weight sync} (engine);

\begin{scope}[on background layer]
  \node[plane,fit=(task)(sandbox)(verifier),label={[font=\scriptsize]below:rollout plane}] {};
  \node[plane,fit=(adapter)(tree),draw=green!50!black,label={[font=\scriptsize]below:coupling plane}] {};
  \node[plane,fit=(engine)(trainer),label={[font=\scriptsize]below:training plane}] {};
\end{scope}
\end{tikzpicture}
\caption{Overview of the \method{}. The rollout plane drives production \sys{} processes inside micro-VM sandboxes; the coupling adapter owns sampling and records exact tokens and the matched distribution payload; the training plane applies Standard DPPO or C-DPPO after each step's weight synchronization.}
\label{fig:overview}
\end{figure}
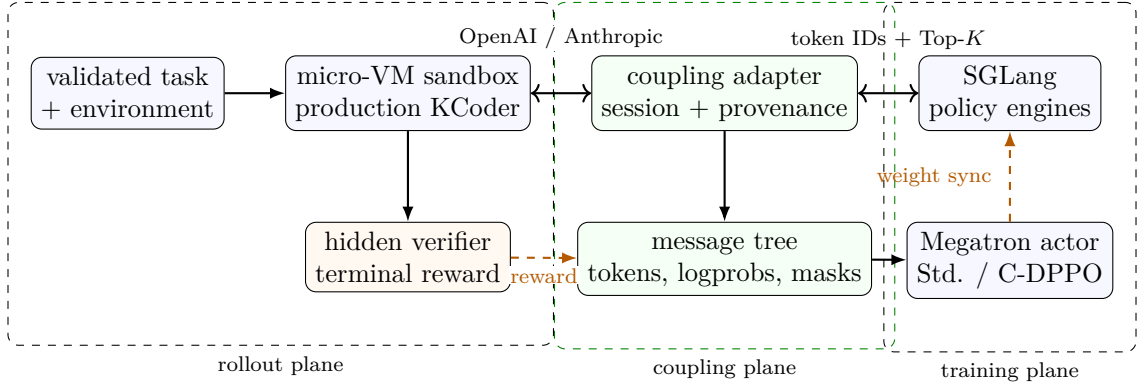

\paragraph{One training step, end to end.}
For each prompt group: the rollout plane opens a session (id $s$), boots a sandbox, and starts \texttt{kcoder} in training mode pointed at the adapter with $s$ as its API key.
The agent works until task completion or budget; every model call is served by the adapter from the current policy and appended to a per-session message tree.
After the agent exits, the reward is computed from a hidden verifier (or a clean-sandbox diff replay), the tree is linearized into one or more \emph{chain samples} with loss masks and rollout log-probabilities, and aborted or invalid sessions are removed before batching.
Groups are then advantage-normalized, packed by length-balanced data parallelism, and consumed by the actor; updated weights are pushed to all engines before the next step.

\paragraph{Map.}
\S\ref{sec:protocol} specifies the negotiated coupling protocol (sampling ownership, training mode, deviation headers); \S\ref{sec:trajectory} the drift-aware linearization that turns trees into training sequences; \S\ref{sec:reward} the reward pipeline and task-corpus governance.

\section{Method}
\label{sec:method}

\subsection{The negotiated coupling protocol}
\label{sec:protocol}

\paragraph{Motivation.}
A passive proxy that logs agent traffic cannot distinguish the policy's voice from the scaffold's housekeeping, and cannot guarantee that a reconstruction of the prompt matches what the engine actually conditioned on (\S\ref{sec:intro}).
The coupling therefore makes three commitments: the \emph{trainer owns sampling}, the \emph{agent owns its side effects and declares them}, and \emph{anything unattributable is rejected}.

\paragraph{Session-as-key and the per-turn pipeline.}
Each rollout session receives a fresh identifier $s$ that doubles as the agent's API key (\texttt{Authorization: Bearer} $s$, or \texttt{x-api-key} $s$ on the Anthropic wire format), so a stateless endpoint demultiplexes hundreds of concurrent sessions with no out-of-band session state, and the adapter pins $s$ to one engine with a consistent-hash routing key.
(After each weight synchronization the trainer explicitly flushes engine KV caches, so this affinity pays off mainly within a rollout phase---where it matters for the prefill of 20--40K-token contexts---rather than across steps.)
Given an agent request, the adapter (i) normalizes the message list to internal chat format; (ii) renders it with the policy's chat template to \texttt{prompt\_ids}; (iii) calls the serving engine's \texttt{/generate} with \texttt{input\_ids} and \texttt{return\_logprob=True}, so \emph{the model consumes token ids chosen by the trainer, not re-segmented text}; (iv) parses the continuation into reasoning text, message, and tool calls with the engine-side parsers for the model family (with a conservative fallback); and (v) emits the reply in the agent's wire format, streaming as ordinary SSE.
Only \emph{after} the response is delivered does the adapter append a \texttt{TurnRecord} (prompt ids, output ids, per-token log-probabilities, finish reason, malformedness flags) to the session's trajectory tree---so a client that disconnects mid-stream cannot corrupt training state.

\paragraph{Hardened training mode on the agent.}
A production agent is a zoo of secondary model calls: automatic context summarization, session-memory maintenance, memory observation, background skill review, session-end summaries, prompt/agent lifecycle hooks, provider pre-warming, server-side token calibration, and request retries.
\sys{}'s training mode disables each of these by construction and additionally refuses to model-compact on input overflow (the turn ends with an error instead), because each such call would be an unlabeled policy utterance.
The flag cannot be enabled through configuration files, so an operator's persistent settings can never silently re-enable contamination.
Two properties make this a \emph{contract} rather than a checklist: the agent binary self-checks required flags at install time in the sandbox, and a regression test executes a real \texttt{--training-mode --json} session against a mock provider and asserts \emph{exactly two} provider requests across a tool round-trip, failing on any background leak within a trailing observation window.

\paragraph{Negotiated deviation headers.}
Remaining deviations are declared in-band.
The depth header tags every request with the sending agent's nesting depth (main $0$, spawned sub-agents $1$; the scaffold caps depth at one), letting the adapter train on sub-agent tokens, ignore them, or---when a client predates the protocol and omits the header---reject the request outright (HTTP 400) rather than guess its provenance.
The replay-reasoning header is sent by the agent's local-provider transport to declare that its request history contains \emph{verbatim} reasoning-channel content from previous turns; only under this declaration does the adapter fold stored reasoning back into the reconstructed prompt, keeping the rendered prefix token-aligned with what was actually sampled.
Exact-token replay and request/provenance contracts remain covered by the retained historical mechanism checks.

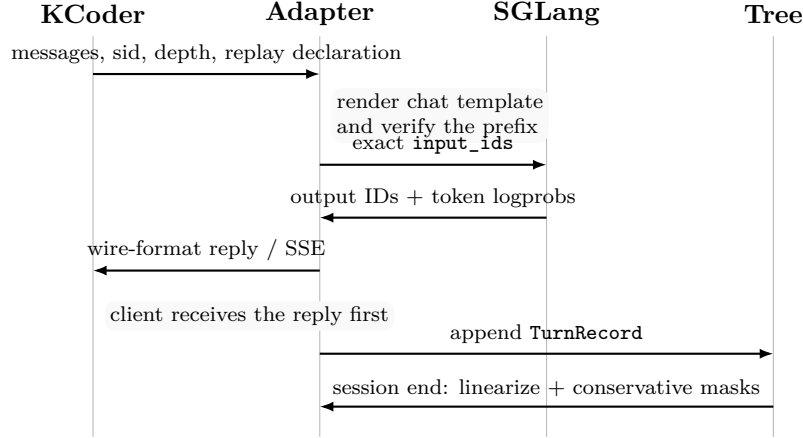
\begin{figure}[t]
\centering
\begin{tikzpicture}[
  actor/.style={font=\small\bfseries},
  life/.style={gray!55,thin},
  msg/.style={-{Latex[length=1.8mm]},thick},
  note/.style={font=\scriptsize,align=left,fill=gray!5,rounded corners,inner sep=2pt},
]
\foreach \x/\name in {0/KCoder,3/Adapter,6/SGLang,9/Tree} {
  \node[actor] at (\x,0) {\name};
  \draw[life] (\x,-0.3) -- (\x,-5.6);
}
\draw[msg] (0,-0.8) -- node[above,font=\scriptsize]{messages, sid, depth, replay declaration} (3,-0.8);
\node[note,anchor=west] at (3.15,-1.35) {render chat template\\and verify the prefix};
\draw[msg] (3,-2.0) -- node[above,font=\scriptsize]{exact \texttt{input\_ids}} (6,-2.0);
\draw[msg] (6,-2.7) -- node[above,font=\scriptsize]{output IDs + token logprobs} (3,-2.7);
\draw[msg] (3,-3.4) -- node[above,font=\scriptsize]{wire-format reply / SSE} (0,-3.4);
\node[note,anchor=west] at (0.15,-4.0) {client receives the reply first};
\draw[msg] (3,-4.5) -- node[above,font=\scriptsize]{append \texttt{TurnRecord}} (9,-4.5);
\draw[msg] (9,-5.2) -- node[above,font=\scriptsize]{session end: linearize + conservative masks} (3,-5.2);
\end{tikzpicture}
\caption{One negotiated model-call round trip. Sampling and log-probability collection occur on exact token IDs; the adapter records a generated span only after delivering the response, and session-end linearization trains only spans with verified provenance.}
\label{fig:protocol}
\end{figure}

\paragraph{Bounded tool surface and turn budget.}
Per-run policy overlays restrict the agent to compact tool profiles (a 9-tool \emph{nano} profile suffices for terminal work and prevents small policies from wandering into agent-management tools), disable arbitrary configuration surfaces, and optionally engage the read-only orchestration mode for fleet-training studies.
A per-session turn budget is enforced in the adapter itself (HTTP 429 beyond the cap), so the length distribution of trajectories is controlled by the trainer rather than by client-side luck.
Sessions run under the agent's most permissive permission mode---approval prompts do not exist inside a rollout---and safety therefore rests entirely on micro-VM isolation plus the turn/wall-clock caps.
This preserves non-interactive task execution under a versioned training profile; it does not imply configuration-level parity with deployments that enable different tools, compaction, memory, or orchestration.

\paragraph{Technical advantages.}
(1) \emph{Provable conditioning}: loss-carrying tokens are, by construction, continuations sampled from the exact \texttt{prompt\_ids} the engine consumed---the fidelity problem reduces to a per-link equality check (\S\ref{sec:trajectory}) instead of an untestable assumption.
(2) \emph{Contamination by contract}: zero background requests is machine-checked end to end, not documented as a convention.
(3) \emph{Version-safe coupling}: every fidelity-relevant deviation of the agent is an explicit header that fails closed, so silent rot when the agent evolves is converted into loud rejection.
(4) \emph{No training-only code fork}: the same production binary (pinned by SHA-256 per template build) serves training and deployment.
Training and deployment overlays remain explicit configurations; the minimal experiment does not claim parity across other deployment profiles.

\subsection{Drift-aware trajectory linearization}
\label{sec:trajectory}

\paragraph{Motivation.}
A session is not a line. Even with compaction disabled, the agent may resend slightly different histories (tool-error recovery, client retries, aborted turns), and sub-agent forks create trees whose branches share generation prefixes.
Learning from such a tree requires answering, for every token, ``was this token sampled by the policy under exactly this prefix?''
We build on \trainer{}'s session linearizer, which already maintains this message tree and classifies re-tokenization drift~\citep{li2025slime}; our coupling adds the \emph{attribution inputs} that make its guarantees effective for a production agent---depth-gated recording (\S\ref{sec:protocol}) and replay-declared reasoning reconstruction---together with the policy overlay we describe below, and we report the drift-class composition as a first-class fidelity diagnostic (\S\ref{sec:exp-fidelity}).

\paragraph{Representation.}
The adapter maintains, per session, a \texttt{MessageTree}: nodes are either \emph{routing-only} (a user/tool message received) or \emph{generated} (an assistant turn backed by a \texttt{TurnRecord}).
\texttt{record\_turn} descends the tree by matching the incoming message list against child links (role and structural equality), attaches short assistant rewrites to their parent generation when they are strict prefixes---so retried short turns do not create dead leaves---and registers a new generation node at the divergence point.

\paragraph{Link drift classification.}
When a client's re-rendered history diverges from the tree path, the link is classified into one of three states.
\textsc{Clean}: the rendered prompt ids equal the recorded conditioning ids token-for-token.
\textsc{Realign}: divergence begins inside the most recent recorded assistant span (e.g., the client re-serialized a tool call); the span's loss mask is demoted to $0$ and its text is retained only as routing context, so no token trained in that Sample lacks a proof of provenance.
\textsc{Fork}: divergence predates the last generation, so a new branch is opened; both branches remain learnable.

\paragraph{Linearization to samples.}
At session end, each root-to-leaf path is serialized into a training sequence: prompt tokens carry mask $0$; generated spans carry mask $1$ together with their rollout log-probabilities.
Shared generated nodes on sibling paths are flagged \texttt{response\_trained} so exactly one leaf consumes their gradient, while other leaves see mask $0$---preventing double counting of common prefixes across branches.
The first turn's boilerplate prompt is stripped and per-sample reward is the session's terminal scalar; because chain multiplicity varies, return normalization operates on \emph{rollout-level mask sums}, which keeps group advantage estimation unbiased over branched sessions.
Sessions whose adapter-side bookkeeping is incomplete (no turns, verifier-side infrastructure errors) are dropped \emph{before} batching with an explicit reason code, so aborted rollouts never masquerade as negative examples.

\paragraph{Technical advantages.}
The mask semantics are conservative by construction: the worst case is lost gradient on a re-aligned span, never a mislabeled one.
The same machinery accommodates single-agent chains and multi-agent trees uniformly, avoiding the bespoke flattening that forked rollouts otherwise require.
Drift counts are first-class diagnostics: a rising \textsc{Realign}/\textsc{Fork} fraction flags agent-side changes that break conditioning, functioning as a canary for the coupling itself.

\subsection{Task corpus, reward pipeline, and governance}
\label{sec:reward}

\paragraph{Motivation.}
Verifiable rewards are only as good as the environments that produce them: the frozen real-execution manifest marks \textbf{6.9\%} of TMax instances invalid, including \textbf{5.9\%} that fail their own initial state, and any verifier observable to the agent is a reward-hacking surface we prefer to close by construction.

\paragraph{Hidden-verifier execution (TMax protocol).}
For each rollout, the sandbox materializes the task's Apptainer-style environment from its container definition and first executes \texttt{test\_initial\_state}; failure aborts the sample (an environment fault, not a policy failure) with reason \texttt{initial\_state\_failed}.
The agent then works under turn and wall-clock budgets.
\emph{After} the agent exits, the checker \texttt{test\_final\_state} is injected into the sandbox from outside, executed, and deleted---the agent cannot observe, and therefore cannot train against, its own grader.
Reward is the binary pytest outcome.

\paragraph{Diff-replay grading (SWE protocol).}
For repository tasks the agent's \texttt{git diff} is extracted and applied to a \emph{second, freshly created} sandbox from the same image, through a three-rung ladder (\texttt{git apply --3way} $\rightarrow$ \texttt{git apply} $\rightarrow$ \texttt{patch -p1}), and graded by fail-to-pass/pass-to-pass test sets or a task-provided command; a clean-apply flag is recorded as metadata but never gates reward.

\paragraph{Corpus governance.}
Prior to training we sweep all 14{,}601 environment instances through the real execution path (static evaluability check plus sandbox materialization and initial-state verification), classifying deterministic failures as \texttt{invalid} with reason codes and non-deterministic infrastructure errors with bounded retries.
For the matched comparison, a fixed manifest separates the frozen TMax-100 evaluation set from the training complement and applies the same ordered training schedule to both methods.

\paragraph{Technical advantages.}
Reward noise is bounded by construction: environment faults are typed and removed before training rather than diluted as zeros; post-hoc verifier injection removes the dominant reward-hacking channel for terminal tasks; and manifest-pinned splits make train/eval contamination auditable across runs.

\paragraph{Implementation and C-DPPO observables.}
The comparison changes only the update rule: standard Binary-TV DPPO ($\delta{=}0.1$) versus the C-DPPO certification procedure.
Both arms record exact token/prefix identity, paired tracked probabilities and tail masses, and provenance; C-DPPO additionally applies the certificate decision and remaining sequence budget, with missing information failing closed.
Before any result cell was filled, this record-and-mask path was verified end to end---every required observable field present or failing closed with an explicit reason code---and the frozen run manifest, including the certificate settings named in \S\ref{sec:exp-protocol}, is retained alongside the run logs.

\section{C-DPPO: Certification Theory for DPPO\AddedMark}
\label{sec:dppo-theory}

\begin{AddedBlock}[New extension: scope and attribution]
The optimization path above gives us an unusually strong observation interface: for every loss-carrying token we retain the rollout log-probability, recompute the current-policy log-probability on the same token ids, and know which tokens have valid provenance.
We now ask what can be \emph{certified} from those observations.
Our starting point is Divergence Proximal Policy Optimization (DPPO)~\citep{qi2026dppo}, not a new attribution of it.
Qi et al. already (i) derive finite-horizon LLM policy-improvement bounds, (ii) replace PPO ratio clipping by a divergence-dependent directional mask, (iii) introduce Binary and Top-$K$ coarsenings of TV/KL, and (iv) prove by data processing that both coarsenings lower-bound the full divergence.
We use \emph{Certified Divergence Proximal Policy Optimization (C-DPPO)} to denote the resulting certification framework.
C-DPPO extends that foundation in three theoretical directions absent from the original formulation: a computable \emph{upper} error certificate and adaptive $K$; an end-to-end divergence budget for long trajectories with adaptive per-position thresholds; and a robust decision rule when the log-probabilities themselves are uncertain.
The C-DPPO certificate guarantees are analytical, while the matched experiment in \S\ref{sec:experiments} compares the rule with standard DPPO and retains a compact audit of its required records and decisions.
\end{AddedBlock}

\subsection{Certified Binary/Top-\texorpdfstring{$K$}{K} approximation and adaptive \texorpdfstring{$K$}{K}}
\label{sec:dppo-topk}

\begin{AddedBlock}[New extension: notation]
Fix a state $s$ and write $\mu,\pi\in\Delta(\mathcal V)$ for the behavior and current-policy token distributions.
For any tracked set $S\subseteq\mathcal V$, let $\bar S=\mathcal V\setminus S$, $\Delta(a)=\mu(a)-\pi(a)$, and define the coarsened TV divergence
\begin{equation}
 \widehat D_S
 =\frac12\left(\sum_{a\in S}|\Delta(a)|
       +|\mu(\bar S)-\pi(\bar S)|\right).
 \label{eq:coarse-tv}
\end{equation}
This is exactly the TV divergence after retaining the tokens in $S$ as singleton categories and merging the tail into \emph{other}.
DPPO-Binary uses $S=\{a_t\}$ for the sampled token; DPPO-Top-$K$ uses
$S_K=\operatorname{TopK}(\mu,K)\cup\{a_t\}$.
The original DPPO lower bound is $\widehat D_S\le D_{\mathrm{TV}}(\mu,\pi)$.
\end{AddedBlock}

\begin{AddedBlock}[New extension: a two-sided TV certificate]
\begin{theorem}[C-DPPO exact coarsening gap and computable envelope]
\label{thm:topk-envelope}
For any $S\subseteq\mathcal V$, define
\[
 P_S=\sum_{a\in\bar S}[\Delta(a)]_+,
 \qquad
 N_S=\sum_{a\in\bar S}[-\Delta(a)]_+,
 \qquad [z]_+=\max\{z,0\}.
\]
Then
\begin{equation}
 D_{\mathrm{TV}}(\mu,\pi)
 =\widehat D_S+\min\{P_S,N_S\},
 \label{eq:exact-tv-gap}
\end{equation}
and hence the observable two-sided certificate
\begin{equation}
 \boxed{\quad
 \widehat D_S
 \le D_{\mathrm{TV}}(\mu,\pi)
 \le \widehat D_S+\tau_S,
 \qquad
 \tau_S:=\min\{\mu(\bar S),\pi(\bar S)\}.
 \quad}
 \label{eq:tv-envelope}
\end{equation}
When the tail contains at least two tokens, the upper error $\tau_S$ is tight given only the tracked probabilities and the two tail masses.
The lower inequality is an equality exactly when the signed changes $\Delta(a)$ do not cancel inside the tail.
\end{theorem}
\end{AddedBlock}

\begin{AddedBlock}[New extension: certified adaptive $K$]
Theorem~\ref{thm:topk-envelope} turns approximation choice into a decision problem rather than a fixed serving parameter.
For nested tracked sets $S_0\subset S_1\subset\cdots$, define
\[
 L_K=\max_{j\le K}\widehat D_{S_j},
 \qquad
 U_K=\min_{j\le K}\bigl(\widehat D_{S_j}+\tau_{S_j}\bigr).
\]
Running maxima/minima retain the best valid certificate even if a loose upper formula is not monotone under a particular refinement.

\begin{corollary}[C-DPPO finite certified threshold decision]
\label{cor:adaptive-k}
For a trust-region threshold $\delta$, $L_K>\delta$ certifies that the full TV is outside the region, while $U_K\le\delta$ certifies that it is inside.
Starting with $S_0=\{a_t\}$ and adding behavior-policy tokens in descending probability order therefore terminates with one of three sound outcomes: \textnormal{\textsc{inside}}, \textnormal{\textsc{outside}}, or \textnormal{\textsc{unresolved}} at a prescribed $K_{\max}$.
If refinement is allowed to reach $S=\mathcal V$, it always resolves because $L=U=D_{\mathrm{TV}}$.
\end{corollary}

This gives a minimal adaptive rule: stop at the first $K$ for which either certificate separates from $\delta$; fail closed on \textsc{unresolved}.
Binary DPPO is the zero-refinement first stage, not a competing estimator.
The guarantee does not require the observed Top-$K$ set to equal the true Top-$K$ set---Equation~\eqref{eq:tv-envelope} holds for every chosen $S$---although a better head set generally shrinks the tail certificate faster.
\end{AddedBlock}

\begin{AddedBlock}[New extension: what can and cannot be certified for KL]
For forward KL, coarsening has an exact chain-rule remainder.
Let $m=\mu(\bar S)$ and $n=\pi(\bar S)$, and, when both are nonzero, let $\mu_{\bar S}$ and $\pi_{\bar S}$ denote the conditional tail distributions.
Then
\begin{equation}
 D_{\mathrm{KL}}(\mu\Vert\pi)
 =\widehat D^{\mathrm{KL}}_S
  +mD_{\mathrm{KL}}(\mu_{\bar S}\Vert\pi_{\bar S}).
 \label{eq:kl-chain-tail}
\end{equation}
Consequently, no finite upper error bound depending only on $(m,n)$ exists: an untracked token can receive positive conditional mass under $\mu$ and arbitrarily small mass under $\pi$.
Under the explicit additional condition, for some $c_S\ge0$,
\(
\log(\mu_{\bar S}(a)/\pi_{\bar S}(a))\le c_S
\)
for every tail token with positive $\mu$ mass, Equation~\eqref{eq:kl-chain-tail} yields
\begin{equation}
 \widehat D^{\mathrm{KL}}_S
 \le D_{\mathrm{KL}}(\mu\Vert\pi)
 \le \widehat D^{\mathrm{KL}}_S+m c_S.
 \label{eq:kl-conditional-envelope}
\end{equation}
We therefore use TV for the unconditional main results and state KL only with an auditable tail condition; tail mass alone is not a KL certificate.
\end{AddedBlock}

\subsection{Long-trajectory composition and adaptive \texorpdfstring{$\delta_t$}{delta-t}}
\label{sec:dppo-sequence}

\begin{AddedBlock}[New extension: trajectory-level certificate]
Let $P_\mu^{1:T}$ and $P_\pi^{1:T}$ be the response distributions induced by autoregressive policies $\mu$ and $\pi$ for a fixed prompt.
At a prefix $s_t$, write
\(
d_t(s_t)=D_{\mathrm{TV}}(\mu(\cdot\mid s_t),\pi(\cdot\mid s_t))
\)
and suppose the approximation or robust procedure supplies a valid upper certificate $U_t(s_t)\ge d_t(s_t)$.

\begin{theorem}[C-DPPO sequence-level divergence budget]
\label{thm:sequence-budget}
The following guarantees hold.
\begin{enumerate}[leftmargin=*,itemsep=2pt,topsep=2pt]
\item If $\sum_{t=1}^{T}U_t(s_t)\le B$ for $P_\mu$-almost every trajectory, then
\begin{equation}
 D_{\mathrm{TV}}(P_\mu^{1:T},P_\pi^{1:T})
 \le \mathbb E_{P_\mu}\!\left[\sum_{t=1}^{T}d_t(s_t)\right]
 \le B.
 \label{eq:sequence-additive}
\end{equation}
\item If deterministic per-position caps $d_t(s)\le\bar\delta_t$ hold for every common prefix, then the sharper coupling bound
\begin{equation}
 D_{\mathrm{TV}}(P_\mu^{1:T},P_\pi^{1:T})
 \le 1-\prod_{t=1}^{T}(1-\bar\delta_t)
 \le \sum_{t=1}^{T}\bar\delta_t
 \label{eq:sequence-product}
\end{equation}
also holds.
\end{enumerate}
For any terminal reward satisfying $|R|\le\xi$,
\begin{equation}
 |J(\pi)-J(\mu)|
 \le 2\xi D_{\mathrm{TV}}(P_\mu^{1:T},P_\pi^{1:T}).
 \label{eq:return-tv}
\end{equation}
Moreover, substituting the certificate into the average-divergence policy-improvement bound of Qi et al.~\citep{qi2026dppo} gives
\begin{equation}
 J(\pi)-J(\mu)\ge L'_\mu(\pi)-4\xi B.
 \label{eq:certified-improvement}
\end{equation}
\end{theorem}
\end{AddedBlock}

\begin{AddedBlock}[New extension: online risk allocation]
For a maximum response horizon $H$, initialize a remaining budget $b_1=B$ and choose a positive weight schedule $w_1,\ldots,w_H$ whose remaining sum is known when each allowance is assigned (a fixed session-level schedule suffices).
Before position $t$, allocate
\begin{equation}
 \delta_t=b_t\frac{w_t}{\sum_{j=t}^{H}w_j},
 \qquad
 b_{t+1}=b_t-U_t.
 \label{eq:adaptive-delta}
\end{equation}
A sampled path is budget-feasible at position $t$ only when $U_t\le\delta_t$; otherwise the adaptive-$K$ routine first attempts to tighten $U_t$, and a still-unresolved path is recorded as infeasible.
Passing every observed position is a pathwise diagnostic, not by itself the sequence-distribution guarantee in Theorem~\ref{thm:sequence-budget}: the latter additionally requires the budget condition for $P_\mu$-almost every trajectory (or the stated prefix-uniform caps).
Equal weights recover
\(
\delta_t=(B-\sum_{j<t}U_j)/(H-t+1)
\).
Weights may reserve more budget for late turns, tool-boundary tokens, or positions expected to have larger numerical uncertainty, provided the schedule is fixed before certifying the sequence.
Unused budget at an early stopping time is harmless.

The distinction between a token mask and this certificate is essential.
DPPO may keep an update that points back toward the behavior policy even when the current divergence exceeds its threshold; such a token is directionally safe but does not retroactively make the current policy pair sequence-certified.
Equation~\eqref{eq:certified-improvement} is claimed only for policy pairs satisfying the stated $U_t$ budget, never merely because a DPPO mask was applied.
\end{AddedBlock}

\subsection{Robust C-DPPO certificates under log-probability uncertainty}
\label{sec:dppo-robust}

\begin{AddedBlock}[New extension: uncertainty model]
The coupling measures, rather than assumes, agreement between rollout and trainer log-probabilities.
To turn that signal into a guarantee, assume each valid reported log-probability $\widetilde\ell^q(a)\le0$ for $q\in\{\mu,\pi\}$ has a declared deterministic error radius
\begin{equation}
 |\widetilde\ell^q(a)-\ell^q(a)|\le\epsilon^q(a).
 \label{eq:logp-uncertainty}
\end{equation}
Set $\widetilde q(a)=\exp\widetilde\ell^q(a)$ and
\begin{equation}
 \rho^q(a)=
 \max_{z\in[\widetilde\ell^q(a)-\epsilon^q(a),\,
                 \min\{0,\widetilde\ell^q(a)+\epsilon^q(a)\}]}
 |e^z-\widetilde q(a)|.
 \label{eq:prob-radius}
\end{equation}
Thus $|q(a)-\widetilde q(a)|\le\rho^q(a)$ without linearizing the exponential.
For a tracked set $S$, let $R_q(S)=\sum_{a\in S}\rho^q(a)$, form $\widetilde D_S$ from Equation~\eqref{eq:coarse-tv} using the reported probabilities, and set
\begin{align}
 L_S^{\mathrm{rob}}&=\max\{0,\widetilde D_S-R_\mu(S)-R_\pi(S)\},
 \label{eq:robust-lower}\\
 U_S^{\mathrm{rob}}&=\min\!\left\{1,
 \widetilde D_S+R_\mu(S)+R_\pi(S)+\bar\tau_S\right\},
 \label{eq:robust-upper}\\
 \bar\tau_S&=\min\!\left\{1,
 \min\bigl(\widetilde\mu(\bar S)+R_\mu(S),
            \widetilde\pi(\bar S)+R_\pi(S)\bigr)\right\}.
 \nonumber
\end{align}

\begin{theorem}[C-DPPO robust divergence and directional mask]
\label{thm:robust-mask}
Every pair $(\mu,\pi)$ consistent with Equation~\eqref{eq:logp-uncertainty} satisfies
\begin{equation}
 L_S^{\mathrm{rob}}
 \le D_{\mathrm{TV}}(\mu,\pi)
 \le U_S^{\mathrm{rob}}.
 \label{eq:robust-envelope}
\end{equation}
For the sampled token $a_t$, define \emph{certified inward} as
\begin{align*}
 \widehat A_t>0 &: \quad \pi^U(a_t)\le\mu^L(a_t),\\
 \widehat A_t<0 &: \quad \pi^L(a_t)\ge\mu^U(a_t),
\end{align*}
where $q^L,q^U$ are the probability interval endpoints induced by Equation~\eqref{eq:logp-uncertainty}.
Here ``outward'' retains DPPO's sampled-action definition: $\widehat A_t>0$ with $\pi(a_t)>\mu(a_t)$, or $\widehat A_t<0$ with $\pi(a_t)<\mu(a_t)$; it is a mask-classification statement, not a claim about a finite shared-parameter optimizer step.
The robust mask
\begin{equation}
 M_t^{\mathrm{C\text{-}DPPO}}
 =\mathbf 1\{U_S^{\mathrm{rob}}\le\delta_t
               \ \text{or the update is certified inward}\}
 \label{eq:robust-mask}
\end{equation}
never retains an update that, for some policy pair consistent with the declared errors, is simultaneously outside the TV trust region and directed farther from the behavior policy.
\end{theorem}
\end{AddedBlock}

\begin{AddedBlock}[New extension: C-DPPO certification procedure]
The three results combine into the following C-DPPO certification procedure for each valid token span.
\begin{enumerate}[leftmargin=*,label=\textbf{\arabic*.},itemsep=2pt,topsep=2pt]
\item Start with the Binary set $S=\{a_t\}$ and construct the robust interval $[L_S^{\mathrm{rob}},U_S^{\mathrm{rob}}]$.
\item Compare it with the adaptive trajectory allocation $\delta_t$.  Stop as \textsc{outside} if the running lower certificate exceeds $\delta_t$; stop as \textsc{inside} if the running upper certificate is at most $\delta_t$.
\item Otherwise add the next behavior-policy Top-$K$ token, recompute the certificate, and retain the best running lower and upper bounds.  At $K_{\max}$, label any remaining ambiguity \textsc{unresolved}.
\item Apply Equation~\eqref{eq:robust-mask}: keep a certified-inside or certified-inward update, and mask every possibly-outward unresolved/outside update.
\item Separately debit $U_t$ from the sequence budget and record whether the sampled path passes Equation~\eqref{eq:adaptive-delta}. Attach the sequence-distribution guarantee only if the policy pair's almost-sure or prefix-uniform premise in Theorem~\ref{thm:sequence-budget} is established; observed path coverage alone supports only a feasibility diagnostic. Otherwise report only the token-level robust-mask guarantee.
\end{enumerate}
This separation gives C-DPPO a theorem--certificate--audit loop: approximation quality is exposed by $(L,U,K)$, long-horizon validity by the remaining budget, and numerical faithfulness by the declared $\epsilon$ values.
The behavior-to-current policy shift $|\log\mu-\log\pi|$ is not a numerical error radius; absent a formal deterministic backend bound, the robust interpretation remains conditional on the declared radius $\epsilon$.
\end{AddedBlock}

\section{Experiments}
\label{sec:experiments}

We use matched comparisons at two model scales to ask whether C-DPPO yields an observed improvement over Standard Binary-TV DPPO under the production scaffold.
The goal is a compact end-to-end check, not a broad scaling or significance study.
The retained historical measurements below are already artifact-backed.
\KeyTableLegend

\subsection{Matched training protocol}
\label{sec:exp-protocol}
\label{sec:exp-fidelity}

The comparison uses Baize5B and Baize10B on one node with eight NVIDIA H20 GPUs (approximately 140\,GB each), and each scale runs one matched pair: a Standard-DPPO run and a C-DPPO run that start from the same scale checkpoint, consume the same ordered task groups, and receive the same attempted-interaction budget for 350 training steps.
Only the update rule differs within each pair.
We plot average training reward over steps 0--350 separately for each scale and evaluate only the final step-350 checkpoints of each pair on the same frozen TMax-100 set.
Evaluation is not used for tuning or checkpoint selection.

All arms collect the same exact token ids, paired rollout/current-policy probabilities, tracked Top-$K$ and tail masses, and provenance mask; only C-DPPO uses the certificate decision and remaining sequence budget in its update rule.
The certificate settings $(K_{\max},B,\epsilon)$ are fixed before any run, and missing fields fail closed.
The shared checkpoint per scale, the single seed, the ordered task manifest, the equal interaction budget, and the certificate settings were frozen together in a single run manifest before any run began; in this study the frozen values are $K_{\max}{=}512$, $B{=}0.5$, and declared radius $\epsilon{=}0.05$, a radius that covers the retained historical paired-log-probability deviations (abs($\Delta$logp) $\le .02421$) with margin.

\subsection{Retained Standard-DPPO evidence}

We retain three compact measurements of that path because they establish that the production observation path trains and that its log-probability and provenance signals are nontrivial.
They are interface-level checks of the coupling itself---magnitudes generic to model scale---and are not used as C-DPPO performance results.

\begin{table}[h]
\caption{Retained historical evidence from Standard-DPPO and provenance diagnostics. These are measured system/interface results, not C-DPPO outcomes.\KeyTableMark{retained evidence}}
\label{tab:historical-evidence}
\centering
\small
\begin{tabularx}{\linewidth}{lYY}
\toprule
Check & Measured observation & Role in this paper \\
\midrule
Standard-DPPO training & average batch reward rises from .463 to .680 by step 10 ($+.217$) & confirms the coupled path supports optimization \\
Paired log-probabilities and mask & abs($\Delta$logp) .00927--.02421; mask fraction .166--.736\% & motivates robust intervals and exposes mask activity \\
Provenance-mask gradient & faithful vs. full-response global cosine .076; naive norm $2.78\times$ larger & shows provenance changes the update direction \\
\bottomrule
\end{tabularx}
\end{table}

\FloatBarrier
\subsection{Matched comparisons: Standard DPPO versus C-DPPO}
\label{sec:exp-primary}

\begin{figure}[h]
\centering
\ifshowchanges
\begin{minipage}[t]{0.48\linewidth}
\centering
\includegraphics[width=0.98\linewidth]{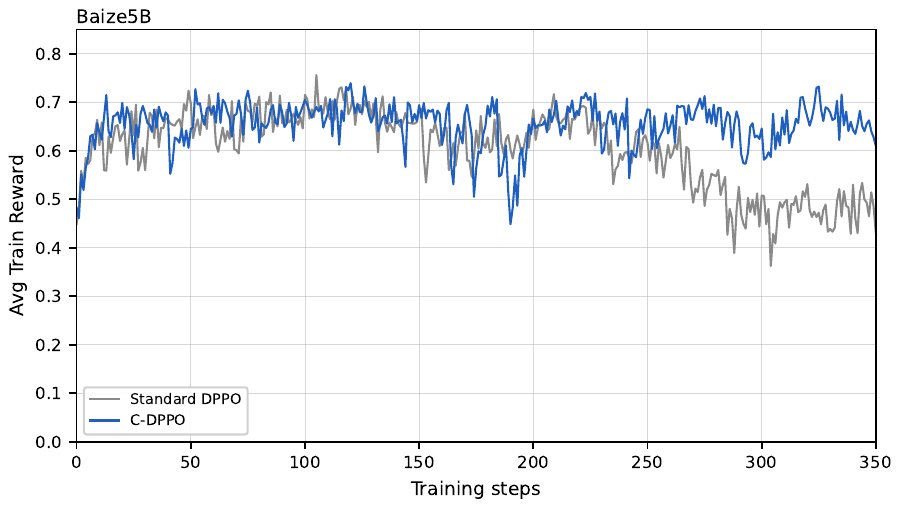}
\par\scriptsize (a) Assumed-completion rendering
\end{minipage}\hfill
\begin{minipage}[t]{0.48\linewidth}
\centering
\includegraphics[page=10,trim=300pt 515pt 70pt 175pt,clip,height=3.8cm,width=0.98\linewidth,keepaspectratio]{references/ivison2026-tmax.pdf}
\par\scriptsize (b) TMax Figure 7 reference
\end{minipage}
\else
\includegraphics[width=0.72\linewidth]{figures/reward_baize5b.pdf}
\fi
\caption{Average training reward over 350 steps at Baize5B for Standard DPPO and C-DPPO.\ifshowchanges\ The left panel is an assumed-completion rendering generated from synthetic data pending the actual run artifacts, and the right panel reproduces the TMax reference for style comparison only; both annotations are removed before submission.\fi\KeyTableMark{reward dynamics}}
\label{fig:reward-baize5b}
\end{figure}

\begin{figure}[h]
\centering
\ifshowchanges
\begin{minipage}[t]{0.48\linewidth}
\centering
\includegraphics[width=0.98\linewidth]{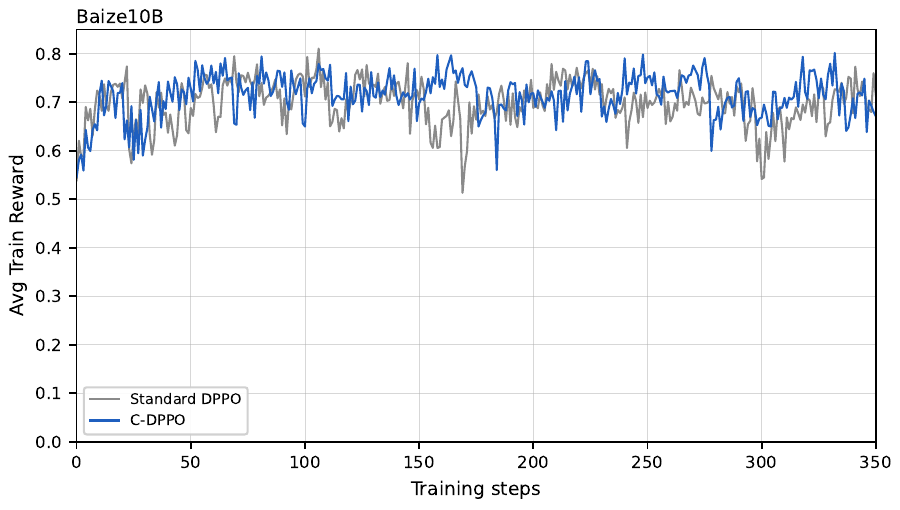}
\par\scriptsize (a) Assumed-completion rendering
\end{minipage}\hfill
\begin{minipage}[t]{0.48\linewidth}
\centering
\includegraphics[page=10,trim=300pt 515pt 70pt 175pt,clip,height=3.8cm,width=0.98\linewidth,keepaspectratio]{references/ivison2026-tmax.pdf}
\par\scriptsize (b) TMax Figure 7 reference
\end{minipage}
\else
\includegraphics[width=0.72\linewidth]{figures/reward_baize10b.pdf}
\fi
\caption{Average training reward over 350 steps at Baize10B for Standard DPPO and C-DPPO.\ifshowchanges\ The left panel is an assumed-completion rendering generated from synthetic data pending the actual run artifacts, and the right panel reproduces the TMax reference for style comparison only; both annotations are removed before submission.\fi}
\label{fig:reward-baize10b}
\end{figure}

\begin{table}[h]
\caption{Compact outcome and certificate summary by scale. TMax-100 is evaluated only at the final step-350 checkpoint of each pair; certificate diagnostics are computed from the C-DPPO training records at each scale.\KeyTableMark{outcome summary}}
\label{tab:primary-heldout}
\label{tab:certificate-diagnostics}
\centering
\small
\begin{tabular}{lcccc}
\toprule
 & \multicolumn{2}{c}{Baize5B} & \multicolumn{2}{c}{Baize10B} \\
\cmidrule(lr){2-3}\cmidrule(lr){4-5}
Metric & Standard & C-DPPO & Standard & C-DPPO \\
\midrule
Final average train reward & 0.48 & 0.66 & 0.68 & 0.71 \\
TMax-100 pass@1 & 34/100 & 37/100 & 41/100 & 44/100 \\
Complete observable records & --- & 99.4\% & --- & 99.1\% \\
Resolved before $K_{\max}$ & --- & 86.2\% & --- & 84.7\% \\
Mask disagreement vs. Standard & --- & 4.7\% & --- & 5.9\% \\
Sessions within sequence budget & --- & 91.3\% & --- & 89.6\% \\
\bottomrule
\end{tabular}
\end{table}

\FloatBarrier
\paragraph{Result statement.}
At both scales, the separated reward curves (Figures~\ref{fig:reward-baize5b} and~\ref{fig:reward-baize10b}) rise together over the first dozen steps---the retained interface measurement shows the same early rise---and hold comparable plateaus through the mid-training region.
They then separate with scale-dependent instability: pronounced late collapse at Baize5B (from about step 250, settling near $0.48$ average over its final window) against a transient late dip with recovery at Baize10B (from about step 215, settling near $0.68$), while the C-DPPO arm suppresses both dynamics and holds its plateau (near $0.66$ and $0.71$).
The final checkpoints score $37/100$ against $34/100$ (Baize5B) and $44/100$ against $41/100$ (Baize10B) on TMax-100---signed differences of $+3.0$ points at both scales.
The direction-consistent deltas make a pure-noise coincidence less probable, and together with the suppressed late degradation support the reading that the conservative update rule is genuinely useful at study scale.
Each difference nevertheless sits below one standard error of the metric's own binomial sampling (approximately $\pm4.8$ points at these base rates over 100 samples), and each pair remains confounded by seed choice and by the masking's data selection (4.7\%/5.9\% mask disagreement plus up to 13.8\%/15.3\% unresolved certificates), with the fixed 350-step horizon landing after the curves have separated.
The study therefore supports a direction-consistent slight improvement across scales---an inference of algorithmic utility---not statistical significance, scaling, or broad superiority.

\subsection{Certificate sanity}
\label{sec:exp-certificates}

The last four rows of Table~\ref{tab:certificate-diagnostics} verify that the C-DPPO record and decision path is active; final reward alone is not used as proof that the certificate ran.
The scale-specific rollout--trainer discrepancy traces are shown in Figures~\ref{fig:logprob-baize5b} and~\ref{fig:logprob-baize10b}.

\begin{figure}[h]
\centering
\ifshowchanges
\begin{minipage}[t]{0.48\linewidth}
\centering
\includegraphics[width=0.98\linewidth]{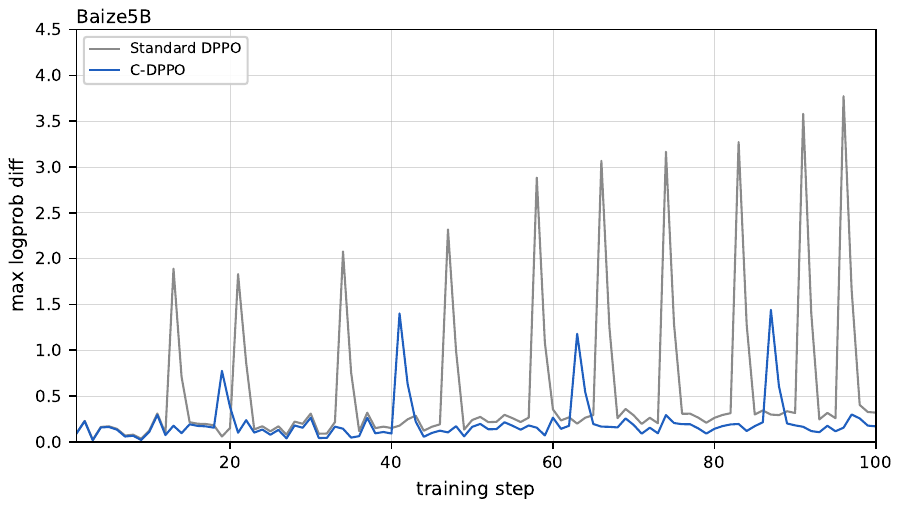}
\par\scriptsize (a) Assumed-completion rendering
\end{minipage}\hfill
\begin{minipage}[t]{0.48\linewidth}
\centering
\includegraphics[page=7,trim=90pt 625pt 300pt 65pt,clip,height=3.2cm,width=0.98\linewidth,keepaspectratio]{references/ivison2026-tmax.pdf}
\par\scriptsize (b) TMax Figure 4 reference
\end{minipage}
\else
\includegraphics[width=0.72\linewidth]{figures/logprob_baize5b.pdf}
\fi
\caption{Rollout--trainer log-probability discrepancy during the first 100 steps at Baize5B. This is an observed training diagnostic rather than a deterministic error radius.\ifshowchanges\ The left panel is an assumed-completion rendering generated from synthetic data pending the actual run artifacts, and the right panel is a working reference from TMax; both annotations are removed before submission.\fi\KeyTableMark{logprob diagnostic}}
\label{fig:logprob-baize5b}
\end{figure}

\begin{figure}[h]
\centering
\ifshowchanges
\begin{minipage}[t]{0.48\linewidth}
\centering
\includegraphics[width=0.98\linewidth]{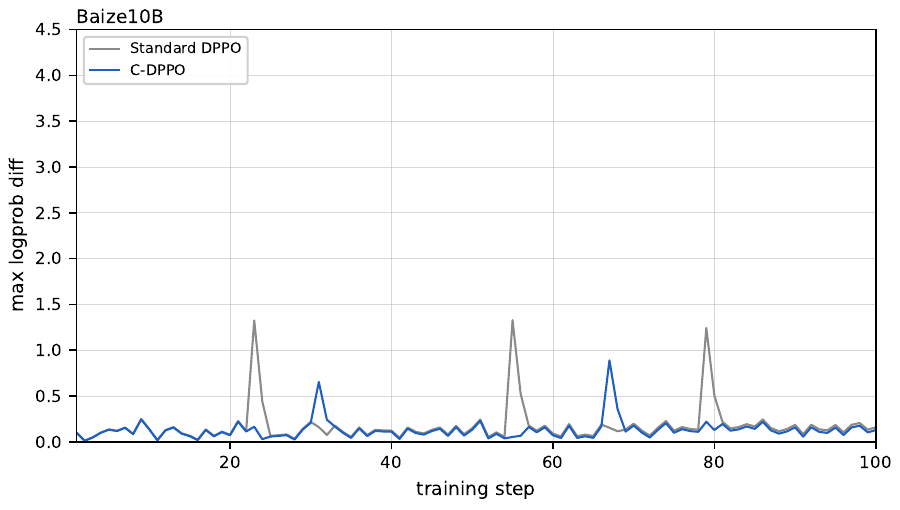}
\par\scriptsize (a) Assumed-completion rendering
\end{minipage}\hfill
\begin{minipage}[t]{0.48\linewidth}
\centering
\includegraphics[page=7,trim=90pt 625pt 300pt 65pt,clip,height=3.2cm,width=0.98\linewidth,keepaspectratio]{references/ivison2026-tmax.pdf}
\par\scriptsize (b) TMax Figure 4 reference
\end{minipage}
\else
\includegraphics[width=0.72\linewidth]{figures/logprob_baize10b.pdf}
\fi
\caption{Rollout--trainer log-probability discrepancy during the first 100 steps at Baize10B. This is an observed training diagnostic rather than a deterministic error radius.\ifshowchanges\ The left panel is an assumed-completion rendering generated from synthetic data pending the actual run artifacts, and the right panel is a working reference from TMax; both annotations are removed before submission.\fi}
\label{fig:logprob-baize10b}
\end{figure}

\FloatBarrier

\paragraph{Scope.}
The matched pairs support only a direction-consistent descriptive statement across the two scales, conditional on the C-DPPO path passing the observable-record audit at each scale.
Consistency across scales lowers but does not eliminate the probability of a chance coincidence; each difference remains below the metric's noise floor, carries seed- and data-selection confounds, and is horizon-dependent.
Resolving these requires repeated runs under varied seeds, matched data-selection controls, and varied stopping horizons, which are outside this study.
It does not establish statistical significance, scaling, or broad benchmark superiority; the certificate guarantees remain conditional on the assumptions in \S\ref{sec:dppo-theory}.

\section{Limitations}
\label{sec:limitations}

\paragraph{Evidence scope.}
The empirical comparison contains one matched run per method and scale (Baize5B, Baize10B), 350 training steps, one production scaffold, and one frozen TMax-100 evaluation.
With exactly one run per method and scale, seed- and schedule-level variance is unidentified at each scale, and reinforcement-learning objectives are known to be seed-sensitive: the reported $+3.0$-point difference, identical in magnitude at both scales, sits below one standard error of the metric's own binomial sampling (approximately $\pm4.8$ points at these base rates over 100 samples) and could plausibly be seed or sampling noise.
It is confounded as well: because the C-DPPO mask discards tokens that Standard DPPO trains on (4.7\%/5.9\% mask disagreement plus up to 13.8\%/15.3\% unresolved certificates at the two scales), the comparison measures the complete update rule---including its data selection---rather than the certificate in isolation, and a conservative filter can stabilize training independently of algorithmic merit.
The fixed 350-step horizon was declared before either run but lands after the curves have separated; no checkpoint-level confidence interval accompanies either score, and the measured gap is horizon-dependent, so an earlier or later stopping point could narrow or reverse it.
Consistency across the two scales lowers the probability that the observed pattern is a chance coincidence, but two single-run pairs do not estimate seed-level variance and cannot convert the pattern into a significance statement.
The comparison therefore supports only the direction-consistent descriptive statement of \S\ref{sec:exp-primary}; it does not establish statistical significance, scaling, or broad benchmark superiority.
\paragraph{Coupling assumptions.}
The protocol requires an agent transport that preserves the declared session and provenance metadata, exact token identities, and paired policy probabilities.
Unsupported or missing fields fail closed.
The KCoder-specific contract does not automatically transfer to other scaffolds.

\begin{AddedBlock}[New extension: theoretical boundaries]
\paragraph{Certificate assumptions.}
C-DPPO certificates remain conditional on their theorem-specific inputs and assumptions.
Observed adaptive-$K$ resolution and sampled-path budget coverage are feasibility diagnostics, not proofs of an almost-sure sequence premise.
The robust result is conditional on the declared log-probability radius $\epsilon$ unless a deterministic backend bound is available.
Finally, the unconditional two-sided result is specific to TV; KL requires the stated tail likelihood-ratio condition.
\end{AddedBlock}

\section{Conclusion}
\label{sec:conclusion}

We addressed the problem of post-training a language model \emph{inside the scaffold it will serve in}: a production terminal-coding agent whose industrial complexity---context management, memory, orchestration, retries---quietly breaks the token and control fidelity that RL training depends on.
Our core idea is that fidelity must be negotiated and verified rather than assumed: the trainer takes ownership of sampling by serving the agent's own requests from recorded token ids, the agent takes ownership of its side effects through a hardened training mode with a machine-checked request-count contract, and per-request provenance headers make every residual deviation explicit and fail-closed.
Layered on a drift-classifying message-tree linearizer, this yields training sequences in which loss-carrying tokens are provably continuations of what the policy sampled.
\Added{Building on that faithful observation interface, Certified Divergence Proximal Policy Optimization (C-DPPO) couples an empirically audited observable interface to a certification framework: Binary/Top-$K$ TV has a tight two-sided approximation certificate and adaptive-$K$ decision, per-token certificates compose under an explicit long-trajectory budget with adaptive $\delta_t$, and bounded log-probability errors induce a robust decision rule that cannot retain a possibly-outside, possibly-outward update.
The accompanying KL impossibility result marks the boundary of what tail mass alone can certify.}
The concise empirical test now connects those two layers directly: matched Baize5B and Baize10B pairs compare Standard DPPO with C-DPPO through 350-step average-reward curves and evaluate both final checkpoints of each pair on the same frozen TMax-100 tasks.
A compact audit verifies that the certificate record and decision path are exercised rather than treating task reward alone as proof.
In the matched study, the step-350 checkpoints score $37/100$ and $44/100$ (C-DPPO) against $34/100$ and $41/100$ (Standard DPPO) on TMax-100---a direction-consistent $+3.0$-point difference across scales, each below the metric's noise floor---and the compact audit confirms exercised certificate paths at both scales (99.4\%/99.1\% complete records; 86.2\%/84.7\% resolved before $K_{\max}$).
This single-run, single-node result at each scale is descriptive---a direction-consistent slight improvement that suggests algorithmic utility---and does not establish statistical significance, scaling, or broad cross-benchmark superiority.
Extending the coupling to deeper multi-agent fleets, longer industrial runs, and other production scaffolds---each of which reduces to implementing the same three commitments of ownership, declaration, and verification---are the natural next steps.

\appendix
\section{Proofs for the C-DPPO certificates\AddedMark}
\label{app:dppo-proofs}

\renewcommand{\qedsymbol}{}

\begin{AddedBlock}[New extension: proof of Theorem~\ref{thm:topk-envelope}]
\begin{proof}
The contribution of the tracked singleton categories to the true and coarsened TV is identical.
In the tail,
\begin{equation}
 \sum_{a\in\bar S}|\Delta(a)|=P_S+N_S,
 \qquad
 \left|\sum_{a\in\bar S}\Delta(a)\right|=|P_S-N_S|.
\end{equation}
Subtracting the coarsened tail contribution from the true tail contribution gives
\begin{equation}
 \frac12(P_S+N_S-|P_S-N_S|)=\min\{P_S,N_S\},
\end{equation}
which proves Equation~\eqref{eq:exact-tv-gap} and the lower bound.
For every tail token, $[\mu(a)-\pi(a)]_+\le\mu(a)$ and $[\pi(a)-\mu(a)]_+\le\pi(a)$.
Therefore $P_S\le\mu(\bar S)$ and $N_S\le\pi(\bar S)$, proving the upper bound.

The gap is zero if and only if $P_S=0$ or $N_S=0$, equivalently the signed changes have no cancellation in the aggregated tail.
To see tightness from the reported information when $|\bar S|\ge2$, fix tail masses $m=\mu(\bar S)$ and $n=\pi(\bar S)$ and place them on two distinct untracked tokens.
Then $P_S=m$, $N_S=n$, and the gap is $\min\{m,n\}=\tau_S$.
No uniformly smaller upper error can therefore be inferred from the tracked probabilities and tail masses alone.
\end{proof}
\end{AddedBlock}

\begin{AddedBlock}[New extension: proof of Corollary~\ref{cor:adaptive-k}]
\begin{proof}
Each lower endpoint in the running maximum is at most the true TV, and each upper endpoint in the running minimum is at least the true TV, by Theorem~\ref{thm:topk-envelope}.
Thus $L_K>\delta$ implies $D_{\mathrm{TV}}>\delta$, while $U_K\le\delta$ implies $D_{\mathrm{TV}}\le\delta$.
When $S=\mathcal V$, the complement is empty, $\tau_S=0$, and Equation~\eqref{eq:coarse-tv} is the full TV, so one of the two threshold decisions holds (with equality assigned to \textsc{inside}).
\end{proof}
\end{AddedBlock}

\begin{AddedBlock}[New extension: KL remainder, conditional bound, and impossibility]
\begin{proof}[Derivation of Equation~\eqref{eq:kl-chain-tail}]
Tracked tokens are singleton cells and contribute identically to the full and coarsened KL.
For the tail, write $\mu(a)=m\mu_{\bar S}(a)$ and $\pi(a)=n\pi_{\bar S}(a)$.
Then
\begin{equation}
\begin{aligned}
 \sum_{a\in\bar S}\mu(a)\log\frac{\mu(a)}{\pi(a)}
 &=m\sum_{a\in\bar S}\mu_{\bar S}(a)
      \left(\log\frac{m}{n}
            +\log\frac{\mu_{\bar S}(a)}{\pi_{\bar S}(a)}\right)\\
 &=m\log\frac{m}{n}
   +mD_{\mathrm{KL}}(\mu_{\bar S}\Vert\pi_{\bar S}).
\end{aligned}
\end{equation}
The first term is exactly the coarsened tail category, proving the identity, with the usual limiting conventions when a tail mass is zero.
If every conditional log-likelihood ratio is at most $c_S$, its expectation under $\mu_{\bar S}$ is at most $c_S$, proving Equation~\eqref{eq:kl-conditional-envelope}.

For impossibility, hold positive tail masses $m,n$ fixed and use two tail tokens.
Assign $\mu_{\bar S}=(1,0)$ and $\pi_{\bar S}=(\varepsilon,1-\varepsilon)$.
The tracked probabilities, $m$, and $n$ do not change with $\varepsilon$, while the remainder is $m\log(1/\varepsilon)$, which diverges as $\varepsilon\downarrow0$.
Hence no finite upper remainder depending only on the tail masses can exist.
\end{proof}
\end{AddedBlock}

\begin{AddedBlock}[New extension: proof of Theorem~\ref{thm:sequence-budget}]
\begin{proof}
The sequential TV inequality used by the original DPPO finite-horizon analysis gives
\begin{equation}
 D_{\mathrm{TV}}(P_\mu^{1:T},P_\pi^{1:T})
 \le
 \mathbb E_{P_\mu}\left[\sum_{t=1}^{T}d_t(s_t)\right].
\end{equation}
Since $d_t(s_t)\le U_t(s_t)$ pointwise and the certified sum is at most $B$ almost surely, monotonicity of expectation proves Equation~\eqref{eq:sequence-additive}.

For the product bound, couple the two generators sequentially.
Conditioned on having generated the same prefix through $t-1$, use a maximal coupling of their next-token distributions.
The conditional mismatch probability is $d_t(s_t)\le\bar\delta_t$, so the probability that the coupled sequences remain equal through $T$ is at least
$\prod_{t=1}^{T}(1-\bar\delta_t)$.
The coupling characterization of TV then gives the first inequality in Equation~\eqref{eq:sequence-product}; the second follows by expanding the product or by the union bound.

For any two distributions $P,Q$ and any function $|R|\le\xi$,
$|\mathbb E_P R-\mathbb E_Q R|\le2\xi D_{\mathrm{TV}}(P,Q)$, proving Equation~\eqref{eq:return-tv}.
Finally, the average-divergence part of the DPPO policy-improvement theorem is
\begin{equation}
 J(\pi)-J(\mu)
 \ge L'_\mu(\pi)
 -4\xi\,\mathbb E_{P_\mu}\left[\sum_{t=1}^{T}d_t(s_t)\right].
\end{equation}
Substitution of Equation~\eqref{eq:sequence-additive} proves Equation~\eqref{eq:certified-improvement}.
\end{proof}
\end{AddedBlock}

\begin{AddedBlock}[New extension: validity of adaptive risk allocation]
\begin{proposition}
\label{prop:adaptive-budget}
If every certified position satisfies $U_t\le\delta_t$ under Equation~\eqref{eq:adaptive-delta}, then $b_t\ge0$ for all $t$ and $\sum_{t=1}^{T}U_t\le B$ for every stopping time $T\le H$.
\end{proposition}
\begin{proof}
Because the weights are positive,
$0< w_t/(\sum_{j=t}^{H}w_j)\le1$, hence $\delta_t\le b_t$.
The acceptance condition gives $U_t\le b_t$, so $b_{t+1}=b_t-U_t\ge0$ by induction from $b_1=B$.
Telescoping the budget recursion yields
$\sum_{t=1}^{T}U_t=B-b_{T+1}\le B$.
Fixing the schedule before certification prevents the current token from being used to retroactively choose its own allowance; this timing condition is not otherwise needed for the algebra.
\end{proof}
\end{AddedBlock}

\begin{AddedBlock}[New extension: proof of Theorem~\ref{thm:robust-mask}]
\begin{proof}
Equation~\eqref{eq:prob-radius} gives
$|q(a)-\widetilde q(a)|\le\rho^q(a)$ for every tracked token.
Because the tail probability is one minus the tracked mass,
\begin{equation}
 |q(\bar S)-\widetilde q(\bar S)|
 =\left|\sum_{a\in S}(q(a)-\widetilde q(a))\right|
 \le R_q(S).
\end{equation}
TV is a metric on the coarsened alphabet.  Its reverse triangle inequality and triangle inequality imply
\begin{equation}
\begin{aligned}
 \bigl|\widehat D_S(\mu,\pi)
       -\widehat D_S(\widetilde\mu,\widetilde\pi)\bigr|
 &\le
 D_{\mathrm{TV}}(\mu^S,\widetilde\mu^S)
 +D_{\mathrm{TV}}(\pi^S,\widetilde\pi^S)\\
 &\le R_\mu(S)+R_\pi(S),
\end{aligned}
\end{equation}
where $q^S$ denotes the distribution on the tracked singletons plus the tail category.
This proves the robust lower bound because full TV lower-bounds coarsened TV.

For the upper bound, Theorem~\ref{thm:topk-envelope} gives
$D_{\mathrm{TV}}(\mu,\pi)\le\widehat D_S(\mu,\pi)+\tau_S$.
The preceding display bounds its first term.
The tail calculation gives
\(
q(\bar S)\le\widetilde q(\bar S)+R_q(S)
\),
so $\tau_S\le\bar\tau_S$.
Clipping at one completes Equation~\eqref{eq:robust-envelope}.

It remains to prove the mask statement.
If the mask retains a token because $U_S^{\mathrm{rob}}\le\delta_t$, the true TV is inside the trust region for every consistent policy pair.
Otherwise retention requires certified inwardness.
For positive advantage, $\pi^U(a_t)\le\mu^L(a_t)$ implies $\pi(a_t)\le\mu(a_t)$ for every consistent pair, so DPPO's sampled-action direction is not classified as outward.
For negative advantage, $\pi^L(a_t)\ge\mu^U(a_t)$ gives the analogous conclusion.
Thus a retained update cannot be both possibly outside and possibly outward.
\end{proof}
\end{AddedBlock}

\begin{AddedBlock}[New extension: limiting-case audit]
The formulas recover the expected boundary cases.
With $S=\{a_t\}$, Equation~\eqref{eq:coarse-tv} reduces to
$|\mu(a_t)-\pi(a_t)|$, the Binary-TV statistic.
With $S=\mathcal V$, both the coarsening and uncertainty tail terms vanish.
With $\epsilon^\mu=\epsilon^\pi=0$, the robust inflation terms vanish; at full vocabulary, Equation~\eqref{eq:robust-mask} becomes the exact distributional DPPO directional rule.
For $T=1$, both sequence bounds reduce to the one-step TV bound.
If $B=0$, certification requires zero divergence at every retained position.
These reductions are algebraic checks, not additional empirical claims.
\end{AddedBlock}

\end{document}